\documentclass[twoside]{article}

\usepackage[accepted]{aistats2023}
\usepackage[utf8]{inputenc} 
\usepackage[T1]{fontenc}    
\usepackage[colorlinks, anchorcolor=white, linkcolor=black, urlcolor=black, citecolor=black]{hyperref} 
\usepackage{url}            
\usepackage{booktabs}       
\usepackage{amsfonts}       
\usepackage{nicefrac}       
\usepackage{microtype}      
\usepackage{xcolor}         
\usepackage{paralist}
\usepackage{comment}
\usepackage{graphicx}
\usepackage{subfigure}
\usepackage{multirow}
\usepackage{amsmath,amsfonts,graphicx,amssymb,bm,amsthm}

\newtheorem{theorem}{Theorem}

\newtheorem{proposition}[theorem]{Proposition}

\usepackage{amsmath,amsfonts,bm}

\def\1{\bm{1}}

\def\vzero{{\bm{0}}}

\def\mI{{\bm{I}}}

\DeclareMathAlphabet{\mathsfit}{\encodingdefault}{\sfdefault}{m}{sl}
\SetMathAlphabet{\mathsfit}{bold}{\encodingdefault}{\sfdefault}{bx}{n}

\def\gN{{\mathcal{N}}}

\def\sR{{\mathbb{R}}}

\newcommand{\E}{\mathbb{E}}

\newcommand{\R}{\mathbb{R}}

\DeclareMathOperator*{\argmin}{arg\,min}

\begin{document}

%

%
\runningauthor{Di He, Shanda Li, Wenlei Shi, Xiaotian Gao, Jia Zhang, Jiang Bian, Liwei Wang, Tie-Yan Liu}

\twocolumn[

\aistatstitle{\textls[-10]{Learning Physics-Informed Neural Networks without Stacked Back-propagation}}

\aistatsauthor{ Di He$^1$\And Shanda Li$^2$\And Wenlei Shi$^3$\And 
Xiaotian Gao$^3$}

\vspace{1mm}

\aistatsauthor{ Jia Zhang$^3$\And Jiang Bian$^3$\And Liwei Wang$^{1,4}$\And Tie-Yan Liu$^3$}

\vspace{1mm}

\aistatsaddress{ $^1$ National Key Laboratory of General Artificial Intelligence,\\
School of Intelligence Science and Technology, Peking University \\
$^2$ Machine Learning Department, School of Computer Science, Carnegie Mellon University\\
$^3$ Microsoft Research $\qquad$ $^4$ Center for Data Science, Peking University} ]

\begin{abstract}
Physics-Informed Neural Network (PINN) has become a commonly used machine learning approach to solve partial differential equations (PDE). But, facing high-dimensional second-order PDE problems, PINN will suffer from severe scalability issues since its loss includes second-order derivatives, the computational cost of which will grow along with the dimension during stacked back-propagation. In this work, we develop a novel approach that can significantly accelerate the training of Physics-Informed Neural Networks. In particular, we parameterize the PDE solution by the Gaussian smoothed model and show that, derived from Stein's Identity, the second-order derivatives can be efficiently calculated without back-propagation. We further discuss the model capacity and provide variance reduction methods to address key limitations in the derivative estimation. Experimental results show that our proposed method can achieve competitive error compared to standard PINN training but is significantly faster. Our code is released at \url{https://github.com/LithiumDA/PINN-without-Stacked-BP}.
\end{abstract}

\section{INTRODUCTION}
Partial Differential Equations (PDEs) play a prominent role in describing the governing physical laws underlying a given system. Finding the solution of PDEs is important in understanding and predicting the physical phenomena from the laws. Recently, researchers sought to solve PDEs via machine learning methods by leveraging the power of deep neural networks \cite{khoo2017solving,han2018solving,long2018pde,long2019pde,raissi2019physics,sirignano2018dgm}. One of the seminal works in this direction is the Physics-Informed Neural Networks (PINN) approach \cite{raissi2019physics}. PINN parameterizes the solution as a neural network and learns the weights by minimizing some loss functional related to the PDEs, e.g., the PDE residual. 

Although the framework is general to learn any PDEs, few previous works experimented with PINN on high-dimensional second-order PDE problems. By thorough investigation, we find PINN training suffers from a significant scalability issue, mainly resulting from stacked back-propagation. 
Note that the PDE residual loss contains second-order derivatives. To update the weights of the neural network by gradient descent, one must first perform automatic differentiation (i.e.,  back-propagation) multiple times to compute the derivatives in the PDE and then calculate the loss. For high-dimensional second-order PDEs, the computational cost in such stacked back-propagation grows along with increasing input dimension \cite{pang2020efficient, meng2021estimating}. This will result in considerable inefficiency, making the PINN approach impractical in large-scale settings. Since some fundamental PDEs, such as Hamilton-Jacobi-Bellman Equations, are high-dimensional second-order PDEs, addressing the scalability issue of PINN becomes essential.

In this paper, we take a first step to tackle the scalability issue of PINN by developing a novel approach to train the model without stacked back-propagation. Particularly, we parameterize the PDE solution $u(x;\theta)$ as a Gaussian smoothed model, $u(x;\theta)=\E_{\delta\sim \gN(0,\sigma^2 \mI)}f(x+\delta;\theta)$, where $u$ transforms arbitrary base network $f$ by injecting Gaussian noise into input $x$. This transformation gives rise to a key property for $u$ where its derivatives to the input can be efficiently calculated \emph{without back-propagation}. Such property is derived from the well-known Stein's Identity \cite{stein1981estimation} that essentially tells that the derivatives of any Gaussian smoothed function $u$ can be reformulated as some expectation terms of the output of its base $f$, which can be estimated using Monte Carlo methods. 

To be concrete, given any PDE problem, we can replace the derivative terms in the PDE with Stein's derivative estimators, calculate the (estimated) residual losses in the forward pass, and then update the weight of the parameters via one-time back-propagation. Our method can accelerate the training of PINN from two folds of advantages. First, after using Stein's derivative estimators, we no longer need stacked back-propagation to compute the loss, therefore saving significant computation time. Second, since the new loss calculation only requires forward-pass computation, it is quite natural to parallelize the computation into distributed machines to further accelerate the training. 

Another point worth noting for the practical application of this method lies in the model capacity, which is highly related to the choice of the Gaussian noise level $\sigma$. We show that for large $\sigma$, the induced Gaussian smoothed models may not be expressive enough to approximate functions (i.e., learn solutions) with a large Lipschitz constant. Therefore, using a small value of $\sigma$ is usually a better choice in practice. However, a small $\sigma$ will lead to high-variance Stein's derivative estimation, which inevitably causes unstable training. We introduce several variance reduction methods that have been empirically verified to be effective for mitigating the problem. Further experiments demonstrate that, compared to standard PINN training, our proposed method can achieve competitive error but is significantly faster.

\section{RELATED WORKS}
\vspace{-0.2cm}
\subsection{Neural Approximation of PDE Solutions}
Neural approximation approaches rely on governing equations and boundary conditions (or variants) to train neural networks to approximate the corresponding PDE solutions. 
Physics-Informed Neural Networks (PINN) \cite{sirignano2018dgm, raissi2019physics} is one of the typical learning frameworks which constrains the output of deep neural networks to satisfy the given governing equations and boundary conditions. 
The application of PINN includes aerodynamic flows \cite{mao2020physics,yang2019predictive}, power systems \cite{misyris2020physics}, and nano optic \cite{chen2020physics}. Recently, there is also a growing body of works on studying the optimization and generalization properties of PINN. \cite{shin2020convergence} proved that the learned PINN will converge to the solution under certain conditions. \cite{krishnapriyan2021characterizing} proposed to use curriculum regularization to avoid failures during PINN training. Different from the neural operator approaches \cite{lu2019deeponet,li2020fourier}, the neural approximation methods can work in an unsupervised manner, without the need of labeled data generated by conventional PDE solvers.

\subsection{Better Training for Physics-Informed Neural Networks}
Despite the success of using PINN in solving various PDEs, researchers recently observed its training inefficiency in multiple aspects. For example, \cite{jagtap2020adaptive} discussed the architecture-wise inefficiency and introduced an adaptive activation function, which optimizes the network by dynamically changing the topology of the PINN loss function for different PDEs. The most relevant works related to our approach are \cite{sirignano2018dgm} and \cite{chiu2021canpinn}, both of which tried to tackle the inefficiency in automatic differentiation by using numerical differentiation. In \cite{sirignano2018dgm}, a Monte Carlo approximation method is proposed to estimate the numerical differentiation of second-order derivatives for some specific second-order PDEs. Concurrently to our work, \cite{chiu2021canpinn} introduced a method, which combines both auto-differentiation and numerical differentiation in PINN training to trade-off numerical truncation error and training efficiency. Compared to these two works, our designed approach can be applied to general PDEs and provide unbiased estimations of any derivatives without the need of back-propagation in the computation of the loss. Detailed discussions can be found in Section \ref{sec:var_reduction}.


\subsection{Gaussian Smoothed Model}
Injecting Gaussian noise to the input has been popularly used in machine learning to improve model's robustness. \cite{li2019secondorder,cohen2019certified} first used Gaussian smoothed models (a.k.a. randomized smoothing) to provide robustness guarantee when facing adversarial attacks. Since then, the smoothed models with different noise types have been developed for various scenarios \cite{zhai2020macer, yang2020randomized}. Leveraging Stein's Identity for efficient derivative estimation is not entirely new in machine learning. The method is one of the standard approaches in zero-order optimization where the exact derivatives cannot be obtained \cite{flaxman2004online, nesterov2017random,liu2020primer,pang2020efficient}. To the best of our knowledge, there is no previous work using Stein's Identity on Gaussian smoothed models for efficient PINN training.


\begin{figure*}[tb]
    \vspace{-4mm}
    \centering
    \includegraphics[width = 0.85\linewidth]{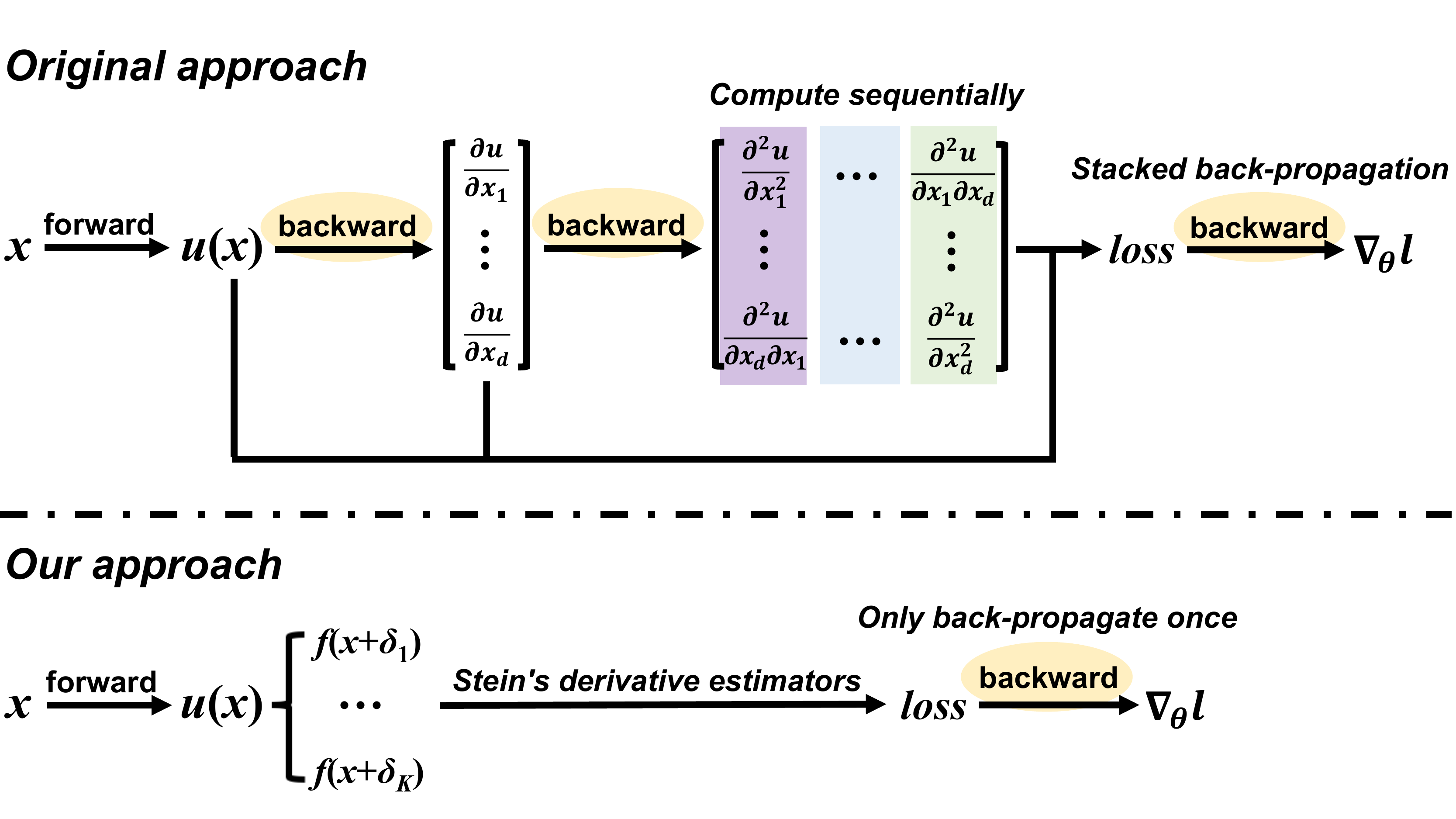}
    \vspace{-3mm}
    \caption{\textbf{Training process of the original PINN approach and our approach}. The original approach requires multiple backward passes to compute the derivative terms in the loss. By contrast, our approach leverages the \textit{Stein's derivative estimators} to compute the loss without back-propagation.}
    \vspace{-1mm}
    \label{fig:teaser}
\end{figure*}

\section{PRELIMINARY}
Without loss of generality, we formulate any partial differential equation as: 
\begin{align}
    \mathcal{L}u(x)&=\phi(x), \quad x\in\Omega\subset\R^d \label{eq:pde_govern}\\
    \mathcal{B}u(x)&=0, \quad x\in\partial\Omega, \label{eq:pde_boundary}
\end{align}
where $\mathcal{L}$ is the partial differential operator and $\mathcal{B}$ is the boundary condition. We use $x$ to denote the spatiotemporal-dependent variable, and use $u$ as the solution of the problem. 

\subsection{PINN Basics}
Physics-informed Neural Network (PINN) \cite{raissi2019physics} is a popular choice to learn the function $u(x)$ automatically by minimizing the Physics-informed loss function induced by the governing equation (\ref{eq:pde_govern}) and boundary condition (\ref{eq:pde_boundary}). 

To be concrete, given any neural network $u(x;\theta)$ with parameter $\theta\in\Theta$, we define the Physics-informed loss function as
\begin{eqnarray}\label{eq:pinn-loss}
    &&l_{\Omega}[\theta]=\| \mathcal{L}u(x;\theta)-\phi(x)\|^2_{L^2(\Omega)} \\
    &&l_{\partial\Omega}[\theta]=\| \mathcal{B}u(x;\theta)\|^2_{L^2(\partial\Omega)}\label{eq:pinn-loss2}.
\end{eqnarray}
The loss term $l_{\Omega}[\theta]$ in Eqn. (\ref{eq:pinn-loss}) corresponds to the PDE residual, which evaluates how $u(x;\theta)$ fits the partial differential equation, and $l_{\partial\Omega}[\theta]$ in Eqn. (\ref{eq:pinn-loss2}) corresponds to the boundary residual, which measures how well $u(x;\theta)$ satisfies the boundary condition. It is easy to see, if there exists $\theta^*$ that achieves zero loss in both residual terms $l_{\Omega}[\theta]$ and $l_{\partial\Omega}[\theta]$, then $u(x;\theta^*)$ will be a solution to the problem. 

To find $\theta^*$ efficiently, PINN approaches leverage gradient-based optimization methods towards minimizing a linear combination of the two losses defined above. As the domain $\Omega$ and its boundary $\partial\Omega$ are usually continuous, Monte Carlo methods are used to approximate $l_{\Omega}[u]$ and $l_{\partial\Omega}[u]$ in practice. As a consequence, the optimization problem will be defined as
\begin{equation}\label{eq:pinn-train1}
    \argmin_{\theta\in\Theta}\hat{l}_{\Omega}[\theta]+\lambda\hat{l}_{\partial\Omega}[\theta].
\end{equation}
In Eqn. (\ref{eq:pinn-train1}), 
$\hat{l}_{\Omega}[\theta]=\frac{1}{N_1}\sum_{i=1}^{N_1}\| \mathcal{L}u(x^{(i)})-\phi(x^{(i)})\|^2_{2}$ and $\hat{l}_{\partial\Omega}[\theta]=\frac{1}{N_2}\sum_{i=1}^{N_2}\| \mathcal{B}u(\tilde x^{(i)}) \|^2_{2}$,
where $\{x^{(1)},\cdots,x^{(N_1)}\}$ and $\{\tilde x^{(1)},\cdots,\tilde x^{(N_2)}\}$ are i.i.d sampled over $\Omega$ and $\partial\Omega$, $N_1$ and $N_2$ are the respective sample sizes, and $\lambda$ is the coefficient used to balance the interplay between the two loss terms.

\subsection{Inefficiency of PINN}
\label{sec:inefficiency_pinn}
It can be seen that the loss for PINN training, as shown in Eqn. (\ref{eq:pinn-train1}), is defined on the network's derivatives. Therefore, the computation of this loss requires multiple back-propagation steps, which can be very inefficient, especially for high-dimensional high-order PDE problems. In details, there are two sources of inefficiency in computing the PINN loss. The first one lies in the order-level inefficiency due to the fact that different orders of derivatives can only be calculated sequentially: One has to first build up the computational graph for the first-order derivatives and then perform back-propagation on this graph to obtain the second-order derivatives. The second source is the dimension-level inefficiency which is a major issue for high-order PDE problems. In modern deep learning frameworks like \texttt{PyTorch} \cite{paszke2019pytorch}, one has to perform back-propagation on the computational graph sequentially for each partial derivative to implement second-order operators like the Laplace operator, leading to a computational cost proportional to the dimensionality of the input \cite{pang2020efficient, meng2021estimating}.


Both of these two sources of inefficiency lead to a non-parallelizable training process, making the learning of PINN for high-dimensional high-order problems particularly slow. It can be easily obtained that for general $k$-order $d$-dimensional PDEs, the computational complexity of each training iteration of PINN using back-propagation is $\mathcal{O}(d^{k-1})$. Therefore, it becomes substantially beneficial to explore new parallelizable methods for PINN training since it may largely increase the training efficiency by making full use of advanced computing hardware, such as GPU.

\section{METHOD}
In this work, we propose a novel method which enables fully parallelizable PINN training. The key of our approach is using a specific formulation of $u(x;\theta)$:
\begin{equation}\label{eq:g-smooth}
    u(x;\theta)=\E_{\delta\sim \gN(0,\sigma^2 \mI)}f(x+\delta;\theta),
\end{equation}
where $f(\cdot;\theta)$ is a neural network with parameter $\theta$, and $\sigma$ is the noise level of the Gaussian distribution. When queried at $x$,  $u$ returns the expected output of $f$ when its input is sampled from a Gaussian distribution $\gN(x,\sigma^2 \mI)$ centered at $x$. It can be easily seen that $u(x;\theta)$ is a ``smoothed'' network constructed from a \textit{``base'' network} $f(x;\theta)$ by injecting Gaussian noise into the input, and we call $u(x;\theta)$ a \textit{Gaussian smoothed model}. For simplicity, we may omit $\theta$ and refer to the base network and the Gaussian smoothed model as $f(x)$ and $u(x)$ in the rest part of the paper.

\subsection{Back-propagation-free Derivative Estimators}
At first glance, using formulation (\ref{eq:g-smooth}) brings more difficulties since the output of $u$ can only be estimated by repeatedly sampling $\delta$. However, we show that the efficiency can be significantly improved during training since all derivatives, as derived by Stein's Identity, can be calculated in parallel without using back-propagation.
\begin{theorem}[\cite{stein1981estimation}]
\label{thm:main}
Suppose $x\in\R^d$. For any measurable function $f(x)$, define $u(x)=\E_{\delta\sim \gN(0,\sigma^2 \mI)}f(x+\delta)$, then we have 
$$\nabla_x u=\E_{\delta\sim \gN(0,\sigma^2 \mI)}[\frac{\delta}{\sigma^2}f(x+\delta)].$$
\end{theorem}
From the above theorem, we can see that the first-order derivative $\nabla_x u$ can be reformulated as an expectation term $\E_{\delta\sim \gN(0,\sigma^2 \mI)}[\frac{\delta}{\sigma^2}f(x+\delta)]$. To calculate the value of the expectation, we can use Monte Carlo method to obtain an unbiased estimation from $K$ i.i.d Gaussian samples, i.e.,
\begin{equation}
    \label{eq:monte-carlo}
    \nabla_x u\approx \frac{1}{K}\sum^K_{k=1}\frac{\delta_k}{\sigma^2}f(x+\delta_k),
\end{equation} 
where $\delta_k\sim \gN(0,\sigma^2 \mI)$, $k=1,...,K$.

It is easy to check that Stein's Identity can be extended to any order of derivatives by recursion. Here we showcase the corresponding identity for Hessian matrix and Laplace operator:
\begin{compactitem}
\item Hessian matrix $\boldsymbol{H}u$: 
\begin{equation*}
    \boldsymbol{H}u =\E_{\delta\sim \gN(0,\sigma^2\mI)}\left[\left(\frac{\delta\delta^{\top}-\sigma^2\mI}{\sigma^4}\right)f(x+\delta)\right].
\end{equation*}

\item Laplace operator $\Delta u$: 
\begin{equation*}
    \Delta u =\E_{\delta\sim \gN(0,\sigma^2\mI)}\left[\left(\frac{\|\delta\|^2-\sigma^2d}{\sigma^4}\right)f(x+\delta)\right].
\end{equation*}
\end{compactitem}

For convenience, we refer to the Monte Carlo estimators based on these identities as \emph{vanilla Stein's derivative estimators}. 

We can plug the Stein's derivative estimators into the physics-informed loss functions (Eqn. \ref{eq:pinn-train1}) defined for a given PDE. In this way, we are able to overcome the two sources of inefficiency in training PINN mentioned earlier: We can see that Stein's derivative estimators for higher-order derivatives no longer require pre-computation for lower-order ones, and the derivatives with respect to each dimension can be obtained in a forward pass \textit{simultaneously} instead of computing them dimension by dimension \textit{sequentially}. Therefore, we can utilize the parallel computing power of GPU resource better and achieve $O(1)$ time complexity. See Figure \ref{fig:teaser} for an illustration of our approach.

These properties of Stein's derivative estimators are appealing because they enable a fully-parallelizable PINN training and lead to significant improvement over efficiency for solving high-dimensional PDEs.  
In the meantime, a natural concern about the deficient expressiveness of Gaussian smooth function may rise. 
In the following subsection, we will take deep discussion on this issue and demonstrate the importance of $\sigma$ in practice to control the model expressiveness.


\subsection{Model Capacity}
\label{sec:model_capacity}
In our method, Gaussian smoothed neural network is used instead of vanilla neural network as the solution of PDE. This modification brings a natural concern, i.e., whether the function space of the proposed Gaussian smoothed models is expressive enough to approximate the solutions of a given PDE problem. In this subsection, we show that the capacity of Gaussian smoothed neural networks is closely related to the Lipschitz function class according to the following theoretical result:
\begin{theorem}
\label{lem:lip}
For any measurable function $f:\R^d \rightarrow \R$, define $u(x)=\E_{\delta \sim \gN(0,\sigma^2 \mI)}f(x+\delta)$, then $u(x)$ is $\frac{F}{\sigma}\sqrt{\frac{2}{\pi}}$-Lipschitz with respect to $\ell_2$-norm, where $F=\sup\limits_{x \in \R^d} |f(x)|$.
\end{theorem}
\begin{proof}
Theorem \ref{thm:main} states that 
\begin{equation*}
    \nabla_x u = \E \left[\frac{\delta}{\sigma^2}f(x+\delta)\right].
\end{equation*}

Thus, for any unit vector $\alpha$, we have 
\begin{equation*}
    |\alpha^{\top} \nabla_x u |\leq \frac{\E\left[|\alpha^{\top} \delta f(x+\delta)|\right]}{\sigma^2} \leq  \frac{F\E\left[|\alpha^{\top} \delta|\right]}{\sigma^2}=\frac{F}{\sigma}\sqrt{\frac{2}{\pi}}.
\end{equation*}

The last equality holds since $\alpha^{\top} \delta\sim\gN(0,\sigma^2)$, which concludes the proof.
\end{proof}

Theorem \ref{lem:lip} can be viewed as a negative result on the capacity of the Gaussian smoothed models. It indicates that the noise level $\sigma$ is an important hyper-parameter that controls the expressive power of the model $u(x)$. For example, if the neural network $f$ uses $\mathrm{tanh}$ activation in the final prediction layer, its output range will be restricted to $(-1,1)$. From Theorem \ref{lem:lip}, it is straightforward to see that the Lipschitz constant of $u$ is no more than $\frac{1}{\sigma}\sqrt{\frac{2}{\pi}}$ no matter how complex the network $f$ is. In this setting, if we have a prior that the solution of a PDE has a large Lipschitz constant, we have to choose a small value of $\sigma$ to approximate it well. On the other hand, using small $\sigma$ would affect the finite-sample approximation error of vanilla Stein's derivative estimators. Without further assumptions on $f$, it is easy to check that the variance of vanilla Stein's derivative estimators in Eqn. (\ref{eq:monte-carlo}) can be inversely proportional to $\sigma$. Thus, naively using Monte Carlo method requires a large sample size for small $\sigma$, which may even slow down the training in practice. In the next subsection, we present several variance reduction approaches that we find particularly useful during training.


\subsection{Variance-Reduced Stein's Derivative Estimators}
\label{sec:var_reduction}
We mainly use two methods to reduce the variance for Stein's derivative estimators, the control variate method and the antithetic variable method. For simplicity, we demonstrate how to apply the two techniques to improve the estimator of $\nabla_x u$ and $\Delta u$, which can be easily extended to other Stein's derivative estimators.

\textbf{The control variate method.} One generic approach to reducing the variance of Monte Carlo estimates of integrals is to use an additive control variate \cite{evans2000approximating, fishman2013monte, hammersley2013monte}, which is known as \emph{baseline}. In our problem, we find $f(x)$ is a proper baseline which can lead to low-variance estimates of the derivative:
\begin{align}
    \nabla_x u=&\E_{\delta\sim \gN(0,\sigma^2 \mI)}\left[\frac{\delta}{\sigma^2}(f(x+\delta)-f(x))\right]\nonumber\\
    \approx&\frac{1}{K}\sum^K_{k=1}\left[\frac{\delta_k}{\sigma^2}(f(x+\delta_k)-f(x))\right]; \label{eqn:control-variates}\\
    \Delta u=&\E\left[\left(\frac{\|\delta\|^2-\sigma^2d}{\sigma^4}\right)(f(x+\delta)-f(x))\right]\nonumber\\
    \approx&\frac{1}{K}\sum^K_{k=1}\left[\left(\frac{\|\delta_k\|^2-\sigma^2d}{\sigma^4}\right)(f(x+\delta_k)-f(x))\right],\label{eqn:control-variates2}
\end{align}
where $\delta_k$ are i.i.d. samples from $\gN(0,\sigma^2 \mI)$. To see clearly why this technique leads to variance reduction, we take Eqn. (\ref{eqn:control-variates}) as an example. we assume $f(x)$ is smooth leverage its Taylor expansion at $x$ to rewrite $\frac{\delta }{\sigma^2}(f(x+\delta)-f(x))$ as $\frac{\delta }{\sigma^2}\left(\delta^{\top}\nabla f(x)+o(\|\delta\|)\right)$. This expression can be further simplified to $\xi\xi^{\top} \nabla f(x)+o(1)$, where $\xi=\delta/\sigma \sim \gN(0,\mI)$. Therefore, the variance of the estimator in Eqn. (\ref{eqn:control-variates}) will be \textit{independent} of $\sigma$. This fact is in sharp contrast to the original estimator provided in Eqn. (\ref{eq:monte-carlo}).
With a similar argument, we can also show that the variance of the estimator in Eqn. (\ref{eqn:control-variates2}) is inversely of proportional to $\sigma$, while the the variance of the original estimator for $\Delta u$ is inversely proportional to $\sigma^2$. 

\textbf{Further improvement using the antithetic variable method.} The antithetic variable method is yet another powerful technique for variance reduction \cite{hammersley1956new}. By using the symmetry of Gaussian distribution, it's easy to see that Eqn. (\ref{eqn:control-variates}) and (\ref{eqn:control-variates2}) still holds when $\delta$ is substituted with $-\delta$, which leads to new estimators. Averaging the new estimator and the one in Eqn. (\ref{eqn:control-variates}) / (\ref{eqn:control-variates2}) gives the following result:
\begin{align}
    \nabla_x u=&\E\left[\frac{\delta}{2\sigma^2}(f(x+\delta)-f(x-\delta))\right]\nonumber\\
    \approx&\frac{1}{K}\sum^K_{k=1}\left[\frac{\delta_k}{2\sigma^2}(f(x+\delta_k)-f(x-\delta_k))\right];\label{eqn:antithetic-variable}
\end{align}
{\small
\begin{align}
    \Delta u=&\E\left[\left(\frac{\|\delta\|^2-\sigma^2d}{2\sigma^4}\right)(f(x+\delta)+f(x-\delta)-2f(x))\right]\nonumber\\
    \approx&\frac{1}{K}\sum^K_{k=1}\left[\left(\frac{\|\delta_k\|^2-\sigma^2d}{2\sigma^4}\right)\right.\cdot \nonumber\\
    &\qquad\quad\left.(f(x+\delta_k)+f(x-\delta_k)-2f(x))\right], \label{eqn:antithetic-variable2}
\end{align}
}

Again, by leveraging the Taylor expansion of $f(x)$, one can show that the variances of the estimators in Eqn. (\ref{eqn:antithetic-variable}) and (\ref{eqn:antithetic-variable2}) are both \textit{independent} of $\sigma$. For example, $\left(\frac{\|\delta\|^2-\sigma^2d}{2\sigma^4}\right)(f(x+\delta)+f(x-\delta)-2f(x))=\left(\frac{\|\delta\|^2-\sigma^2d}{2\sigma^4}\right)(\delta^{\top}\boldsymbol{H}f(x) \delta + o(\|\delta\|^2))$, where $\boldsymbol{H}f(x)$ denotes the Hessian matrix of $f$ at $x$. By letting $\xi=\delta/\sigma \sim \gN(0,\mI)$, the estimator is simplified to $(\|\xi\|^2-d) (\xi^{\top}\boldsymbol{H}f(x) \xi + o(1))$. Thus, its variance is \textit{independent} of $\sigma$.
This property is especially appealing because it enables us to tune model expressiveness according to PDE complexity in practice. For instance, we can use a small $\sigma$ to ensure the model is expressive enough to learn a complex PDE solution. We also provide empirical comparisons between vanilla Stein's derivative estimators and the improved ones. See Section \ref{sec:ablate} for details. 

Note that Eqn. (\ref{eqn:control-variates}) and (\ref{eqn:antithetic-variable2}) look very similar to numerical differentiation in the surface form but they yield substantial differences. First, the roles of the term $f(x)$ are different. In Eqn. (\ref{eqn:control-variates}) and (\ref{eqn:antithetic-variable2}), the term $f(x)$ is introduced as the baseline which doesn't change the value of the expectation since $\E\left[\frac{\delta}{\sigma^2}f(x)\right]=\E\left[\left(\frac{\|\delta\|^2-\sigma^2d}{\sigma^4}\right)f(x)\right]=0$. Therefore, multiplying any constant to $f(x)$ also holds, which will be infeasible for numerical differentiation. Second, our method provides \textit{unbiased} estimation of the derivatives while numerical method can only obtain \textit{biased} derivatives due to truncation error. 

\begin{table*}[t]
\caption{\textbf{Experimental results of solving Poisson's Equation (left) and HJB Equation (right).} ``PINN'' refers to the original PINN approach with stacked auto-differentiation. ``Ours'' refers to our new method using Gaussian smoothed models and Stein's derivative estimators. We report $L^1$ and $L^2$ relative error of the learned solutions, as well as the training time, to compare these two approaches.}
\label{tab:exp-main}
\centering
\vspace{3mm}
\begin{tabular}{ccccc}
    \toprule
    Problem &Method & $L^1$ error & $L^2$ error   & Training time \\ \midrule
    \multirow{2}{*}{\textit{2-dimensional Poisson's Equation}} 
    &PINN   & $0.13\%$ & $0.15\%$ & $18.96$s       \\ 
    &Ours   & $0.16\%$ & $0.19\%$ & $28.44$s       \\ \midrule
    \multirow{2}{*}{\textit{100-dimensional Heat Equation}} 
    &PINN   & $0.52\%$ & $0.60\%$ & $2.35$min       \\ 
    &Ours   & $0.53\%$ & $0.63\%$ & $0.83$min       \\ \midrule
    \multirow{2}{*}{\textit{250-dimensional HJB Equation}} 
    &PINN + adv train  & $0.95\%$ & $1.18\%$ & $38.16$h        \\ 
    &Ours + adv train  & $0.91\%$ & $1.37\%$ & $12.07$h       \\ \bottomrule
\end{tabular}
\end{table*}

\begin{figure*}[ht]
\begin{minipage}{0.3\linewidth}
    \centering
    \includegraphics[width=1\linewidth]{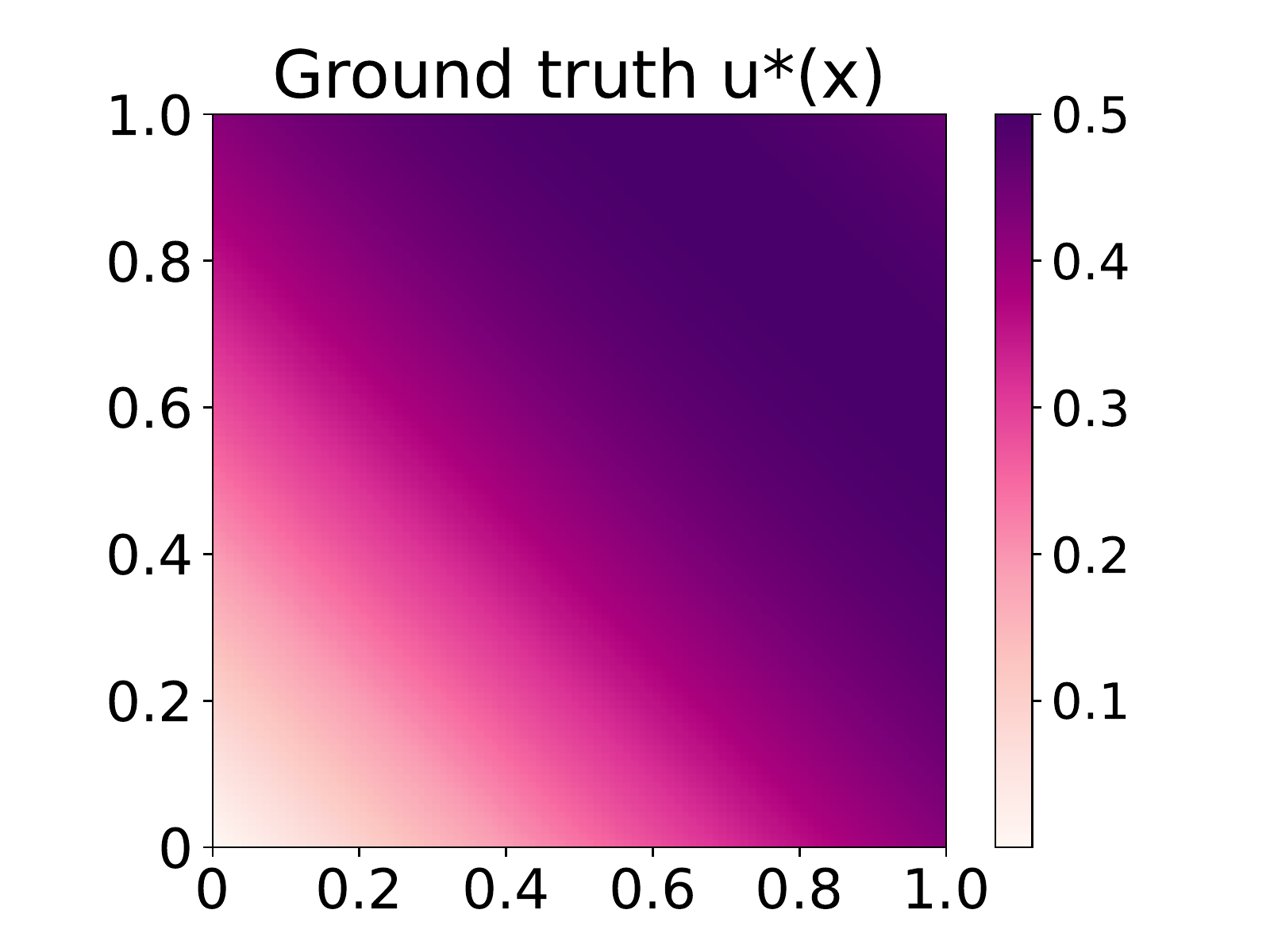}
\end{minipage}\hfill
\begin{minipage}{0.3\linewidth}
    \centering
    \includegraphics[width=1\linewidth]{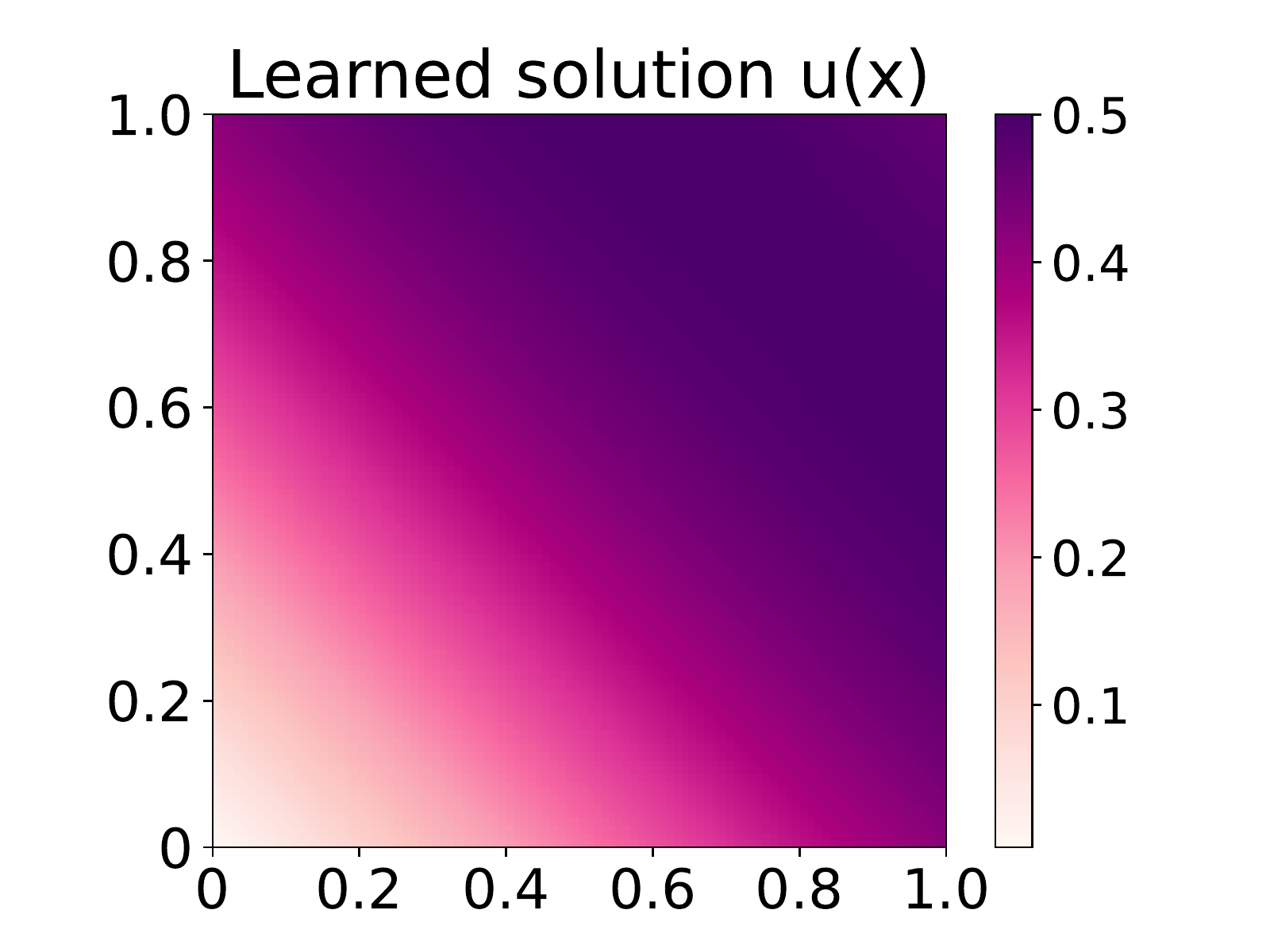}
\end{minipage}\hfill
\begin{minipage}{0.3\linewidth}
    \centering
    \includegraphics[width=1\linewidth]{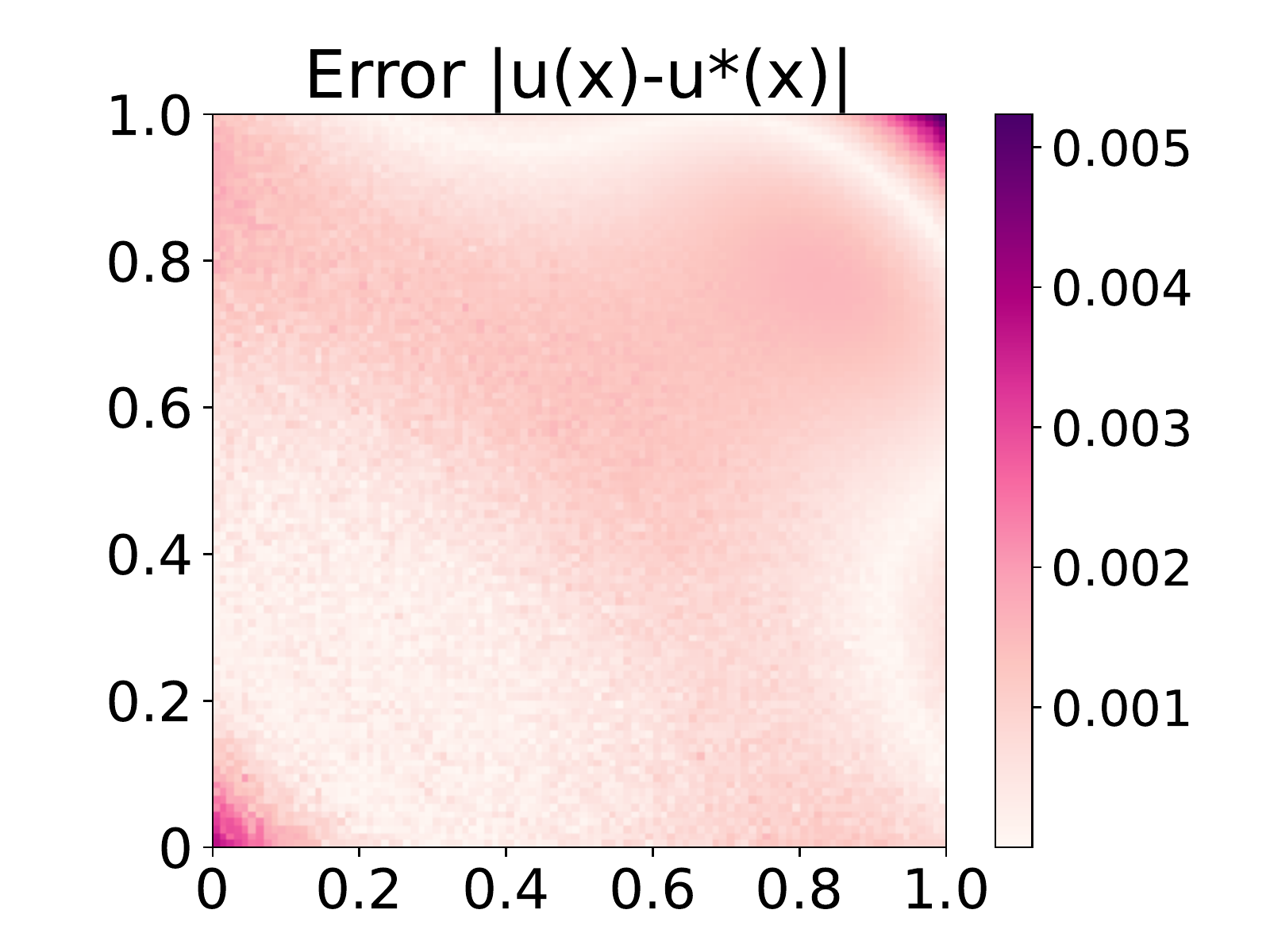}
\end{minipage}
\caption{\textbf{Visualization for two-dimensional Poisson's Equation}. Three subfigures show the ground truth $u^*(x)$, the learned solution $u(x)$ using our method, and the point-wise error $|u(x)-u^*(x)|$ respectively.}
\label{fig:possion}
\end{figure*}

\section{EXPERIMENTAL RESULTS}
\label{sec:exp}

We conduct experiments to verify the effectiveness of our approach on a variety of PDE problems. Ablation studies on the design choices and hyper-parameters are then provided. Our codes are implemented based on \texttt{PyTorch} \cite{paszke2019pytorch}. All the models are trained on one NVIDIA GeForce RTX 2080 Ti GPU with 11GB memory and the reported training time is also measured on this machine. Our code can be found in the supplementary material.

\subsection{Low-dimensional Problems}
\label{sec:poisson}
We first showcase our approach on two-dimensional PDE problems with visualization. In particular, we study the following two-dimensional Poisson's Equation with Dirichlet boundary conditions:
\begin{equation*}
\left\{
\begin{array}{ll}
    \Delta u(x) = g(x) & x\in \Omega \\
    u(x) = h(x) & x \in \partial \Omega
\end{array}
\right.
\end{equation*}

In our experiment, we set $\Omega=[0,1]^2$, $g(x)=-\sin(x_1+x_2)$ and $h(x)=\frac{1}{2}\sin(x_1+x_2)$. This PDE has a unique solution $u^*(x)=\frac{1}{2}\sin(x_1+x_2)$. 

We train a Gaussian-smoothed model to fit the solution with our method. Specifically, the base neural network of our model is a 4-layer MLP with 200 neurons and $\mathrm{tanh}$ activation in each hidden layer. In our Gaussian smoothed model, the noise level $\sigma$ is set to $0.1$, and the number of samples $K$ is set to 2048. We use the variance-reduced derivative estimator in Eqn. (\ref{eqn:antithetic-variable2}) based on the control variate method and the antithetic variable method.
To train the model, we use Adam as the optimizer \cite{kingma2015adam}. The learning rate is set to $3\mathrm{e}-4$ in the beginning and then decays linearly to zero during training. In each training iteration, we sample $N_1=100$ points from the domain $\Omega$ and $N_2=100$ points from the boundary $\partial \Omega$ to obtain a mini-batch. We train the model for 1000 iterations. We use the variance-reduced derivative estimator in Eqn. (\ref{eqn:antithetic-variable}) and (\ref{eqn:antithetic-variable2}) based on the control variate method and the antithetic variable method to compute the loss during training.

We compare our method with original PINN trained with stacked automatic differentiation \cite{raissi2019physics}. The PINN baseline model is trained using the same set of hyperparameters on the same machine to ensure a fair comparison. Evaluations are performed on a hold-out validation set which is unseen during training. We measure the accuracy of the learned solution by calculating the $L^1$ and $L^2$ relative error in the domain $\Omega$. We also report the training time to study the efficiency in practice. The reported results are averaged across 3 different runs.

\paragraph{Experimental Results.} The experimental results are summarized in Table \ref{tab:exp-main}. After training, our model reaches an $L^1$ relative error of $0.16\%$, indicating that the model fits the ground truth well. Furthermore, the final relative error of our model is very close to that of PINN, e.g., $0.13\%$ v.s. $0.16\%$ in terms of $L^1$ relative error. This result demonstrates that using Gaussian smoothed neural networks and Stein's derivative estimators does little harm to the accuracy of the model. Besides, one can notice that the training time of our model is slightly longer than that of PINN. We point out that this is because stacked back-propagation is especially time-consuming for \textit{high dimensional problem}, while this experiment targets a two-dimensional PDE. In this case, removing the need for stacked back-propagation does not necessarily lead to speed-up, and this experiment is not a favorable setting for our method. That being said, this experiment still shows that using Gaussian smoothed neural networks and Stein's derivative estimators is feasible in PINN training.

We also examine the quality of the learned solution $u(x)$ by visualization. Figure \ref{fig:possion} shows the ground truth $u^*(x)$, the learned solution $u(x)$ and the point-wise error $|u(x)-u^*(x)|$. We can see that the learned solution and the ground truth are very similar. Quantitatively speaking, the third subfigure in Figure \ref{fig:possion} shows that the point-wise error of our model is less than $2\mathrm{e}-3$ for most areas, thus the model can approximate the solution to Poisson's Equation with high accuracy when using our method.

\subsection{High-dimensional Heat Equation}
\label{sec:heat}
We consider the high-dimensional Heat Equation to demonstrate the speed and accuracy of our proposed method in solving high-dimensional PDE. Heat Equation is a prototypical parabolic PDE \cite{evans1998partial}, which is deeply connected to several domains including probability theory \cite{lawler2010random}, image processing \cite{aubert2006mathematical}, and quantum mechanics \cite{griffiths2018introduction}. In this experiment, we study the following $N$-dimensional Heat Equation:
\begin{equation*}
\left\{
\begin{array}{ll}
    u_t(x,t)=\Delta u(x,t) & x\in B(0,1), t\in (0,1)\\
    u(x,0) = \|x\|^2/2N & x \in B(0,1) \\
    u(x,t) = t+1/2N & x \in \partial B(0,1), t\in[0,1]
\end{array}
\right.
\end{equation*}

In the above equation, $x\in\sR^N$ is the spatial variable, while $t>0$ is the temporal variable, which is slightly different from the notation in Eqn. (\ref{eq:pde_govern}) where $x$ is denotes as the spatiotemporal-dependent variable in the PDE. We hope this abuse of notations does not confuse the readers. $B(0,1)$ denotes the unit ball in $N$-dimensional space. $\Delta$ denotes the Laplacian operator with respect to the spatial variable $x$. The solution to the above equation is $u(x,t)=t+\|x\|^2/2N$. In our experiments, we focus on the high-dimensional setting and set the problem dimensionality $N$ to $100$. 

Similar to Section \ref{sec:poisson}, We compare our method with original PINN trained with stacked automatic differentiation \cite{raissi2019physics}. The baseline and our model share the same set of hyperparameters, and are trained on the same machine. Evaluations are performed on a hold-out validation set which is unseen during training. We report the accuracy of the learned solution by calculating the $L^1$ and $L^2$ relative error in the domain $B(0,1)\times[0,1]$, as well as the training time.

We train a Gaussian-smoothed model to fit the solution, where the base neural network of our model is a 4-layer MLP with 256 neurons and $\mathrm{tanh}$ activation in each hidden layer. In our Gaussian smoothed model, the noise level $\sigma$ is set to $0.01$, and the number of samples $K$ is set to 2048. We use the variance-reduced derivative estimator in Eqn. (\ref{eqn:antithetic-variable2}) based on the control variate method and the antithetic variable method. The learning rate is set to $1\mathrm{e}-3$ in the beginning and then decays linearly to zero during training. Other training details can be found in the appendix.

\textbf{Experimental Results.}  The experimental results are summarized in the middle of Table \ref{tab:exp-main}. Our model is as accurate as the PINN baseline. For example, our model and the original PINN obtain $0.53\%$ and $0.52\%$ $L^1$ relative error on the test set, respectively. Besides, our model is significantly more efficient than the PINN baseline, achieving $2.83\times$ acceleration in this problem. We emphasize that our model and the baseline are trained for \textit{the same number of iterations}. Therefore, our technique clearly brings noticeable efficiency gains by avoiding stacked back-propagation.

\subsection{High-dimensional Hamilton-Jacobi-Bellman Equation}
\label{sec:hjb}
We further use the high-dimensional Hamilton-Jacobi-Bellman (HJB) Equation to showcase the effectiveness of our proposed method in solving complicated non-linear high-dimensional PDE. HJB Equation is one of the most important non-linear PDE in optimal control theory. It has a wide range of applications, including physics \cite{sieniutycz2000hamilton}, biology \cite{li2011inverse}, and finance \cite{pham2009continuous}. Its discrete-time counterpart is the Bellman Equation widely used in reinforcement learning \cite{sutton2018reinforcement}.

\textbf{Experimental Design.} Following existing works \cite{han2018solving, wang2022is}, we study the classical linear-quadratic Gaussian (LQG) control problem in $N$ dimensions, whose HJB Equation is a second-order PDE\footnote{Similar to the Heat Equation in Section \ref{sec:heat}, $x$ denotes the state variable and $t$ denotes the temporal-dependent variable here.} defined as below:
\begin{equation}
\left\{
\begin{array}{ll}
    u_t + \Delta u -\mu \|\nabla_x u\|^2 = 0 & x \in \sR^N, 
    t \in [0, T] \\
    u(T, x) = g(x) & x \in \sR^N
\end{array}
\right.
\label{eq:lqg}
\end{equation}

As is shown in \cite{han2018solving}, there is a unique solution to Eqn. (\ref{eq:lqg}):
\begin{equation*}
    u(x,t) = -\frac{1}{\mu(2\pi)^{\frac{n}{2}}} \ln\int_{\R^n}  \mathrm{e}^{-\frac{\|y\|^2}{2}-\mu g(x-\sqrt{2(T-t)}y)} \mathrm{d} y.
\end{equation*}

We set the parameters $\mu=1$, $T=1$, and the terminal cost function $g(x)=\ln\left(\dfrac{1+\|x\|^2}{2}\right)$. To evaluate the training speed and performance in high dimensional cases, we set the problem dimensionality $N$ to $250$. 

We compare our method with the sate-of-the-art PINN-based approach on HJB Equation \cite{wang2022is}. Specifically, \cite{wang2022is} shows that original PINN training algorithm cannot guarantee to learn an accurate solution to a large class of HJB Equation, and propose to use adversarial training for PINN to learn the solution with theoretical guarantee. The resulting training algorithm is powerful yet time-consuming because it involves both stacked back-propagation for high-dimensional function and additional computation over-head in adversarial training. In this experiment, we follow \cite{wang2022is} to apply adversarial training, while removing stacked back-propagation with our method. 

We train a Gaussian-smoothed model to fit the solution, where the base neural network of our model is a 4-layer MLP with 768 neurons and $\mathrm{tanh}$ activation in each hidden layer. 
The noise level $\sigma$, the number of samples $K$, and the derivative estimators are the same as those in Section \ref{sec:heat}.
The learning rate is set to $2\mathrm{e}-4$ in the beginning and then decays linearly to zero during training. Other training settings, including the batch size, the total iterations, etc., are the same as those in \cite{wang2022is}, and the details can be found in the appendix.


\textbf{Experimental Results.}  The experimental results are summarized in the bottom of Table \ref{tab:exp-main}, where the performance of the baseline ``PINN + adv train'' is reported in \cite{wang2022is}. It's clear that the accuracy of our model is on par with the state-of-the-art result, showing that using the Gaussian smoothed model and approximated derivatives does not hurt the model performance. Besides, our model and the baseline are trained for the same number of iterations, and the reported training time indicates that our method is much more efficient compared with stacked back-propagation in high-dimensional problem. To be specific, our method is $3.16\times$ faster than the baseline, which largely reduces the training cost.

These observations clearly demonstrates that our method can significantly accelerate the training for high-dimensional PDE problems without sacrificing performance. We believe this is an initial step towards efficiently solving high-dimensional PDEs using deep learning approaches. 
\begin{figure*}[t]
\centering
\begin{minipage}{0.37\linewidth}
    \centering
    \includegraphics[width=0.85\linewidth]{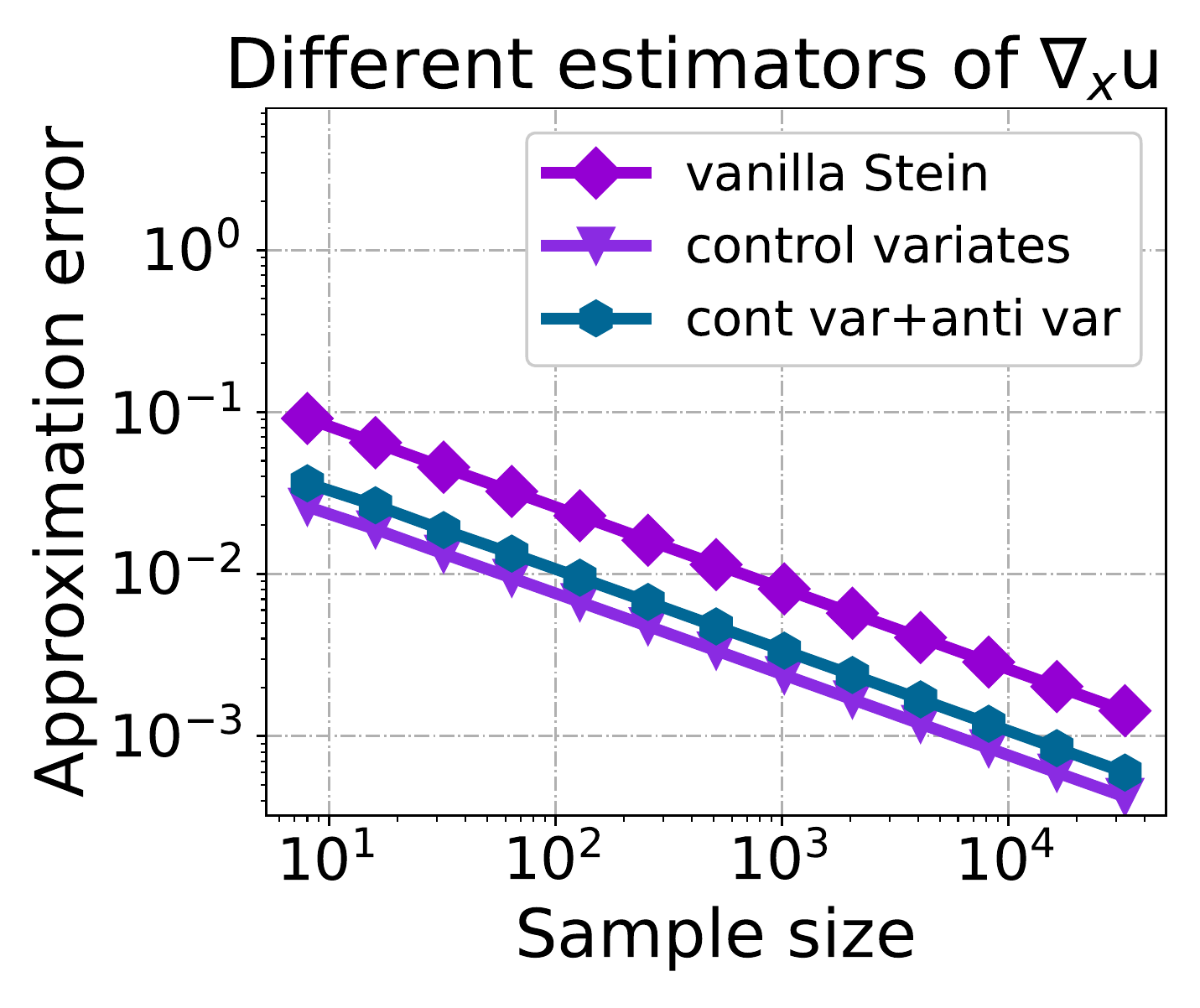}
\end{minipage}
\begin{minipage}{0.37\linewidth}
    \centering
    \includegraphics[width=0.85\linewidth]{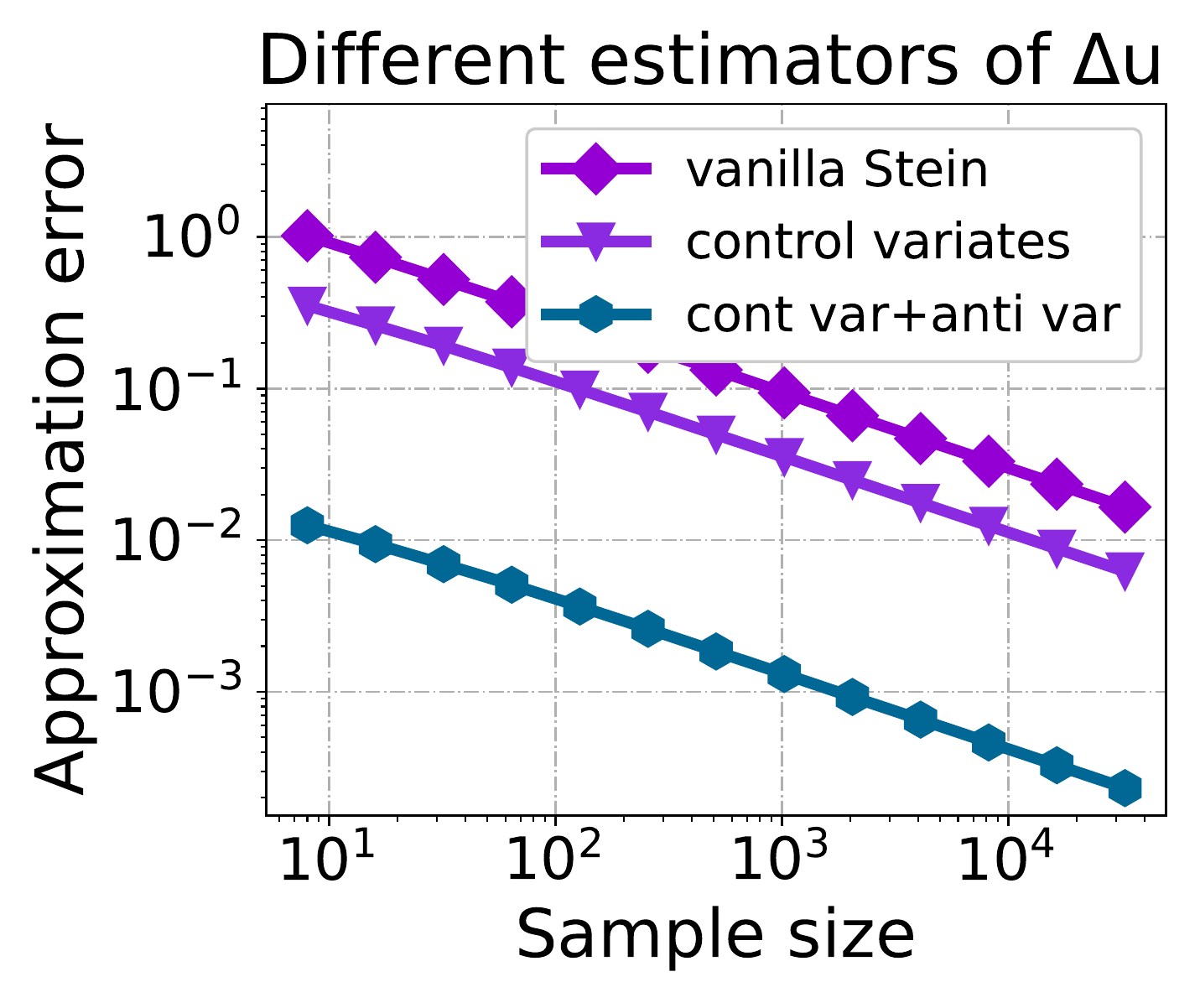}
\end{minipage}
\vspace{-3mm}
\caption{\textbf{Numerical results on the approximation error}. The two subfigures show the approximation error of different estimators for the gradient (left) and the Laplace operator (right), where the $x$-axis indicates the logarithmic-scaled sample sizes and the $y$-axis indicates the logarithmic-scaled approximation error. }
\label{fig:ablate}
\end{figure*}

\subsection{Ablation studies}
\label{sec:ablate}

\begin{table*}[t]

\caption{\textbf{Experimental results for ablation studies.} We experiment on the 250-dimensional HJB Equation using different hyperparameters. The left table investigates the impact of the noise level $\sigma$ on the final performance. The right table compare the performances of our model trained with different sample size $K$ in the Stein's derivative estimator.}
\label{tab:ablate}
\vspace{0.3cm}
\centering
    \begin{minipage}{0.4\linewidth}
        \centering
        \begin{tabular}{ccc}
        \toprule
        Noise level $\sigma$ & $L^1$ error & $L^2$ error  \\ \midrule
        1   & $22.74\%$ & $27.90\%$     \\ 
        $1\mathrm{e}-1$   & $1.04\%$ & $1.42\%$      \\ 
        $1\mathrm{e}-2$   & $0.91\%$ & $1.37\%$      \\ 
        $1\mathrm{e}-3$  & $0.96\%$ & $1.39\%$  \\ \bottomrule
        \end{tabular}
    \end{minipage}
    \begin{minipage}{0.53\linewidth}
        \centering
        \begin{tabular}{cccc}
        \toprule
        Sample size $K$ & $L^1$ error & $L^2$ error   & Training time \\ \midrule
        256   & $1.11\%$ & $1.42\%$ & $1.34$h       \\ 
        512   & $0.98\%$ & $1.32\%$ & $2.63$h       \\ 
        1024   & $0.96\%$ & $1.38\%$ & $5.59$h       \\ 
        2048  & $0.91\%$ & $1.37\%$ & $12.07$h     \\ \bottomrule
        \end{tabular}
    \end{minipage}
\end{table*}

we conduct ablation studies in this section to examine the effects of different design choices.

\textbf{Regarding the choice of derivative estimators.} To compare different statistical estimators introduced in the paper, we conduct several numerical experiments and study the approximation errors of each estimator. We experiment with a Gaussian-smoothed model $u(x)=\E_{\delta\sim \gN(0,\sigma^2 \mI)}f(x+\delta)$, where the base neural network $f(x)$ is a randomly initialized and then fixed 4-layer MLP. The input/output dimension is set to 1000/1 respectively. We set the noise level $\sigma$ to $0.1$. We consider the derivative estimators for the gradient $\nabla_x u$ and the Laplace operator $\Delta u$. We randomly sample $10^3$ points in $[0,1]^{1000}$. For each sampled point $x$, we sample $10^5$ Gaussian noises and use back-propagation to calculate $\nabla_x u$ and $\Delta u$ as an oracle. Empirically, we observe the output of the oracle is very stable, whose variance is less than $1\mathrm{e}-7$. We compute the $L_1$ distance between the derivatives returned by the oracle and the derivative estimators used in the paper as the evaluation metric. We compare three estimators: a) the vanilla Stein's derivative estimators, defined in Eqn. (\ref{eq:monte-carlo}); b) variance-reduced estimators with the control variate method, defined in Eqn. (\ref{eqn:control-variates}) and (\ref{eqn:control-variates2}); c) variance-reduced estimators with both the control variate and the antithetic variable method, defined in Eqn. (\ref{eqn:antithetic-variable}) and (\ref{eqn:antithetic-variable2}). 
For each estimator, we vary the sample size $K$ from $8$ to $32768$.

The results are shown in Figure \ref{fig:ablate}. For the first-order derivative, i.e., $\nabla_x u$, the vanilla Stein's derivative estimator is already accurate, and the variance-reduced ones further improve it slightly. For the second-order derivative, i.e., $\Delta u$, the vanilla Stein's derivative estimator performs poorly, whose error is larger than $1\mathrm{e}-2$ even with sample size $K=32768$. The control variate method slightly improves, but the resulting estimator is still inaccurate given reasonable sample size. Combining the control variate method with the antithetic variable method significantly reduces the approximation error: The corresponding estimator is $100\times$ more accurate than the vanilla one. These observations suggest that variance reduction is essential in using Stein's Identity, especially for estimating high-order derivatives. 

\textbf{Regarding the choice of the noise level $\sigma$.} As discussed in Section \ref{sec:model_capacity}, the noise level $\sigma$ controls the expressive power of the Gaussian-smoothed model $u(x)$, which will greatly affect the performance of the learned models. To understand the influence of $\sigma$ during training, we conduct experiments on 250-dimensional HJB Equation with $\sigma$ ranging from $1\mathrm{e}-3$ to 1. In all these experiments, we set the sample size $K$ to 2048. We tune the learning rate and the loss coefficient in the experiment and report the \textit{best} result for each setting.

The results are shown in the left part of Table \ref{tab:ablate}. Note that changing the noise level $\sigma$ does not affect the computational complexity. Thus, the training time of different models are nearly the same, and we only report the relative error for brevity.  
It can be seen that the learned solutions are accurate when $\sigma\leq 0.1$, and the value of $\sigma$ does not affect the model's performance much. 
However, when $\sigma$ becomes too large, e.g., $\sigma=1$, the relative error of the model is high at the end of training. This indicates that the solution of the PDE may not be in the function class that the model can express. Note that $\sigma$ cannot be arbitrarily small because the calculation of the second-order derivatives involves small values, e.g., $\sigma^4$, which will introduce round-off error when we use float precision.

\textbf{Regarding the choice of sample size $K$.} In our method, the sample size $K$ is a hyper-parameter that controls the speed-accuracy trade-off. When $K$ is larger, the approximation of the derivatives is more accurate, while the training would be slower. We run experiments on 250-dimensional HJB Equation with the sample size $K$ ranging from $256$ to $2048$.  In all these experiments, we set the noise level $\sigma$ to 0.01. We tune the learning rate and the loss coefficient in the experiment and report the \textit{best} result for each setting.

The results are shown in the right part of Table \ref{tab:ablate}.  From the results we can see that a larger sample size $K$ can lead to better performance, at the cost of increased training time. However, the training is not very sensitive to the sample size. Even with only 256 samples, our method is still accurate (and extremely efficient). This is because the variance-reduced Stein's derivative estimators can provide sufficiently accurate estimation given a small sample size.

\section{CONCLUSION}
In this paper, we develop a novel approach that can significantly accelerate the training of Physics-Informed Neural Networks. In particular, we parameterize the PDE solution by the Gaussian smoothed model and show that, as derived from Stein's Identity, the second-order derivatives can be efficiently calculated without back-propagation. Experimental results show that our proposed method can significantly accelerate PINN without sacrificing accuracy. One limitation of our work is that our method only leads to acceleration for high dimensional PDE. Accelerating PINN on more classes of problems can be an exciting direction for future work. 
We believe this work is an initial step towards efficiently solving the high-dimensional partial differential equations using deep learning approaches and will address various other challenges along the way. 

\section*{ACKNOWLEDGMENTS}
We thank Bin Dong and Yiping Lu for their helpful discussions.
This work is supported by National Science Foundation of China (NSFC62276005), The Major Key Project of PCL (PCL2021A12), Exploratory Research Project of Zhejiang Lab (No. 2022RC0AN02), and Project 2020BD006 supported by PKUBaidu Fund.

\bibliography{main}
\bibliographystyle{apalike}

\appendix
\onecolumn

\section{OMITTED PROOFS}
In this section, we always assume $x\in\R^d$, $f(x)$ is a measurable function, and that $u(x)=\E_{\delta\sim \gN(0,\sigma^2 \mI)}f(x+\delta)$.

\subsection{Proof of Theorem 1}
\begin{proof} 
Note that 
\begin{equation}
\label{eq:int-u}
     u(x)=(2\pi)^{-\frac{n}{2}} \int_{\R^d}\mathrm{e}^{-\frac{\|\delta\|^2}{2\sigma^2}} f(x+\delta) \mathrm{d} \delta
=(2\pi)^{-\frac{n}{2}} \int_{\R^d}  \mathrm{e}^{-\frac{\|t-x\|^2}{2\sigma^2}} f(t) \mathrm{d} t.
\end{equation}
   
We have
\begin{align*}
    \nabla_x u(x)&=(2\pi)^{-\frac{n}{2}} \int_{\R^d}  \nabla_x \mathrm{e}^{-\frac{\|t-x\|^2}{2\sigma^2}} f(t) \mathrm{d}t
    =(2\pi)^{-\frac{n}{2}} \int_{\R^d}  \frac{t-x}{\sigma^2} \mathrm{e}^{-\frac{\|t-x\|^2}{2\sigma^2}} f(t) \mathrm{d}t\\
    &=(2\pi)^{-\frac{n}{2}} \int_{\R^d}  \frac{\delta}{\sigma^2} \mathrm{e}^{-\frac{\|\delta\|^2}{2\sigma^2}} f(x+\delta) \mathrm{d}\delta
    =\E_{\delta\sim \gN(0,\sigma^2 \mI)}\left[\frac{\delta}{\sigma^2}f(x+\delta)\right],
\end{align*}
which concludes the proof.
\end{proof}

\subsection{Deviations of the Second Order Stein's derivative estimators}
\begin{proposition}
\label{prop:second-order}
We have
\begin{align}
    \boldsymbol{H}u & =\E_{\delta\sim \gN(0,\sigma^2\mI)}\left[\left(\frac{\delta\delta^{\top}-\sigma^2\mI}{\sigma^4}\right)f(x+\delta)\right];\label{eq:prop-hessian}\\
    \Delta u & =\E_{\delta\sim \gN(0,\sigma^2\mI)}\left[\left(\frac{\|\delta\|^2-\sigma^2d}{\sigma^4}\right)f(x+\delta)\right].\label{eq:prop-laplacian}
\end{align}
\end{proposition}
\begin{proof}
Note that $u(x)$ can be expressed as in Eqn.(\ref{eq:int-u}). Therefore, for $i,j\in\{1,\cdots,d\}$ and $i\neq j$, we have
\begin{align}
    \frac{\partial^2}{\partial x_i \partial x_j} u(x)&=(2\pi)^{-\frac{n}{2}} \int_{\R^d}  \frac{\partial^2}{\partial x_i \partial x_j} \mathrm{e}^{-\frac{\|t-x\|^2}{2\sigma^2}} f(t) \mathrm{d}t =(2\pi)^{-\frac{n}{2}} \int_{\R^d}  \frac{(t_i-x_i)(t_j-x_j)}{\sigma^4} \mathrm{e}^{-\frac{\|t-x\|^2}{2\sigma^2}} f(t) \mathrm{d}t \nonumber\\ 
    &=(2\pi)^{-\frac{n}{2}} \int_{\R^d}  \frac{\delta_i \delta_j}{\sigma^4} \mathrm{e}^{-\frac{\|\delta\|^2}{2\sigma^2}} f(x+\delta) \mathrm{d}\delta =\E_{\delta\sim \gN(0,\sigma^2 \mI)}\left[\frac{\delta_i \delta_j}{\sigma^4}f(x+\delta)\right], \label{eq:prop-partial-ij}
\end{align}

For the case where $i=j$, similar computations will yield
\begin{equation}
    \frac{\partial^2}{\partial x_i^2} u(x)=(2\pi)^{-\frac{n}{2}} \int_{\R^d}  \frac{\partial^2}{\partial x_i^2} \mathrm{e}^{-\frac{\|t-x\|^2}{2\sigma^2}} f(t) \mathrm{d}t
    =\E_{\delta\sim \gN(0,\sigma^2 \mI)}\left[\frac{\delta_i^2-1}{\sigma^4}f(x+\delta)\right],\label{eq:prop-partial-i2}
\end{equation}

By definition of Hessian matrix, Eqn.(\ref{eq:prop-hessian}) can be proven by combining Eqn.(\ref{eq:prop-partial-ij}) and (\ref{eq:prop-partial-i2}).

As for Eqn.(\ref{eq:prop-laplacian}), note that $\Delta u=\mathrm{tr}(\boldsymbol{H}u)$, where $\mathrm{tr}(\cdot)$ denotes the trace of a matrix. Therefore, 
\begin{equation*}
    \Delta u =\E\left[\mathrm{tr}\left(\frac{\delta\delta^{\top}-\sigma^2\mI}{\sigma^4}\right)f(x+\delta)\right]=\E\left[\left(\frac{\|\delta\|^2-\sigma^2d}{\sigma^4}\right)f(x+\delta)\right],
\end{equation*}
which concludes the proof.
\end{proof}

\subsection{Deviations of the Variance-Reduced Stein's derivative estimators}
\begin{proposition}[Estimators based on the control variate method]
\begin{align}
    \label{eqn:control-variates-app}
    &\nabla_x u=\E_{\delta\sim \gN(0,\sigma^2 \mI)}\left[\frac{\delta}{\sigma^2}(f(x+\delta)-f(x))\right]; \\
    \label{eqn:control-variates2-app}
    &\Delta u=\E\left[\left(\frac{\|\delta\|^2-\sigma^2d}{\sigma^4}\right)(f(x+\delta)-f(x))\right].
\end{align}
\end{proposition}

\begin{proof}
Note that $\delta\sim \gN(0,\sigma^2 \mI)$ implies $\E [\delta] = 0$ and $\E[ \|\delta\|^2-\sigma^2d ]= 0$. 

Thus, $\E [\delta f(x)] = \E[\|\delta\|^2-\sigma^2d f(x) ]= 0$. Applying Theorem 1 and Proposition \ref{prop:second-order}, Eqn.(\ref{eqn:control-variates-app}) and (\ref{eqn:control-variates2-app}) hold by linearity of expectation.
\end{proof}

\begin{proposition}[Estimators based on the control variate and antithetic variable method]
\begin{align}
    \label{eqn:antithetic-variable-app}
    \nabla_x u=&\E\left[\frac{\delta}{2\sigma^2}(f(x+\delta)-f(x-\delta))\right];\\
    \label{eqn:antithetic-variable2-app}
    \Delta u=&\E\left[\left(\frac{\|\delta\|^2-\sigma^2d}{2\sigma^4}\right)(f(x+\delta)+f(x-\delta)-2f(x))\right]
\end{align}
\end{proposition}
\begin{proof}
    Note that $\delta\sim \gN(0,\sigma^2 \mI)$ implies $-\delta\sim \gN(0,\sigma^2 \mI)$. Thus, we can substitute $\delta$ with $-\delta$ in Eqn.(\ref{eqn:control-variates-app}) and obtain
    \begin{equation}
        \label{eqn:control-variates-anti}
        \nabla_x u=\E_{\delta\sim \gN(0,\sigma^2 \mI)}\left[-\frac{\delta}{\sigma^2}(f(x-\delta)-f(x))\right]
    \end{equation}
    
    The summation of Eqn.(\ref{eqn:control-variates-app}) and (\ref{eqn:control-variates-anti}) gives Eqn.(\ref{eqn:antithetic-variable-app}). Eqn.(\ref{eqn:antithetic-variable2-app}) is proven similarly.
\end{proof}

\section{EXPERIMENTAL SETTINGS}
\subsection{Poisson's Equation}
\paragraph{Hyperparameters.} The hyperparameters used in our experiment on Poisson's Equation are described in Table \ref{tab:poisson}.

\begin{table}[h]
    \caption{\textbf{Derailed experimental settings} of Poisson's Equation.}
    \label{tab:poisson}
        \begin{center}
            \begin{tabular}{lc}
                \toprule
                \textit{Model Configuration} \\
                \midrule
                \textbf{Layers} & 4 \\
                \textbf{Hidden dimension} & 256  \\
                \textbf{Activation} & $\mathrm{tanh}$ \\
                \textbf{Noise level $\sigma$} & $0.01$ \\
                \textbf{Sample size $K$} & 2048 \\
                \midrule
                \textit{Hyperparameters} \\
                \midrule
                \textbf{Total iterations} & 1000 \\
                \textbf{Domain Batch Size} $N_1$ & 100  \\
                \textbf{Boundary Batch Size} $N_2$ & 100 \\
                \textbf{Boundary Loss Weight} $\lambda$ & 300 \\
                \textbf{Learning Rate} & $1\mathrm{e}-3$ \\
                \textbf{Learning Rate Decay} & Linear\\
                \textbf{Adam $\varepsilon$} & $1\mathrm{e}-8$ \\
                \textbf{Adam($\beta_1$, $\beta_2$)} & (0.9, 0.999) \\
                \bottomrule
            \end{tabular}
    \end{center}
\end{table}

\paragraph{Training data.} The training data is sampled \textit{online}. Specifically, in each iteration, we sample $N_1$ i.i.d. data points, $x^{(1)}, \cdots, x^{(N_1)}$, uniformly from the domain $\Omega$, and $N_2$ i.i.d. data points, $\tilde x^{(1)}, \cdots, \tilde x^{(N_2)}$, uniformly from the boundary $\partial \Omega$.

\subsection{Heat Equation}

\paragraph{Hyperparameters.} The hyperparameters used in our experiment on HJB Equation are described in Table \ref{tab:heat}.

\begin{table}[!t]
    \caption{\textbf{Derailed experimental settings} of Heat Equation.}
    \label{tab:heat}
        \begin{center}
            \begin{tabular}{lc}
                \toprule
                \textit{Model Configuration} \\
                \midrule
                \textbf{Layers} & 4 \\
                \textbf{Hidden dimension} & 256  \\
                \textbf{Activation} & $\mathrm{tanh}$ \\
                \textbf{Noise level $\sigma$} & $0.01$ \\
                \textbf{Sample size $K$} & 2048 \\
                \midrule
                \textit{Hyperparameters} \\
                \midrule
                \textbf{Total iterations} & 1000 \\
                \textbf{Domain Batch Size} $N_1$ & 50  \\
                \textbf{Initial Condition Batch Size} $N_2$ & 50 \\
                \textbf{Spatial Boundary Batch Size} $N_3$ & 50 \\
                \textbf{Initial Condition Weight} $\lambda_2$ & 1000 \\
                \textbf{Spatial Boundary Weight} $\lambda_3$ & 1000 \\
                \textbf{Learning Rate} & $1\mathrm{e}-3$ \\
                \textbf{Learning Rate Decay} & Linear\\
                \textbf{Adam $\varepsilon$} & $1\mathrm{e}-8$ \\
                \textbf{Adam($\beta_1$, $\beta_2$)} & (0.9, 0.999) \\
                \bottomrule
            \end{tabular}
    \end{center}
\end{table}

\paragraph{Training data.} The training data is sampled \textit{online}. Specifically, in each iteration, we \textit{uniformly} sample $N_1$ i.i.d. data points, $(x^{(1)},t^{(1)}), \cdots, (x^{(N_1)},t^{(N_1)}) $, from the domain $B(0,1)\times (0,1)$; $N_2$ i.i.d. data points, $(\tilde x^{(1)},0), \cdots, (\tilde x^{(N_2)},0)$, from $B(0,1) \times \{0\}$; and $N_3$ i.i.d. data points, $(\hat x^{(1)},\hat t^{(1)}), \cdots, (\hat x^{(N_3)},\hat t^{(N_3)}) $, from $\partial B(0,1) \times [0,1]$.

\subsection{HJB Equation}

\paragraph{Hyperparameters.} The hyperparameters used in our experiment on HJB Equation are described in Table \ref{tab:hjb}.

\begin{table}[!h]
    \caption{\textbf{Derailed experimental settings} of HJB Equation.}
    \label{tab:hjb}
        \begin{center}
            \begin{tabular}{lc}
                \toprule
                \textit{Model Configuration} \\
                \midrule
                \textbf{Layers} & 4 \\
                \textbf{Hidden dimension} & 768  \\
                \textbf{Activation} & $\mathrm{tanh}$ \\
                \textbf{Noise level $\sigma$} & $0.01$ \\
                \textbf{Sample size $K$} & 2048 \\
                \midrule
                \textit{Hyperparameters} \\
                \midrule
                \textbf{Total iterations} & 10000 \\
                \textbf{Domain Batch Size} $N_1$ & 50  \\
                \textbf{Boundary Batch Size} $N_2$ & 50 \\
                \textbf{Boundary Loss Weight} $\lambda$ & 500 \\
                \textbf{(Adversarial training) Inner Loop Iterations $K$} & 20 \\
                \textbf{(Adversarial training) Inner Loop Step Size $\eta$} & 0.05 \\
                \textbf{Learning Rate} & $2\mathrm{e}-4$ \\
                \textbf{Learning Rate Decay} & Linear\\
                \textbf{Adam $\varepsilon$} & $1\mathrm{e}-8$ \\
                \textbf{Adam($\beta_1$, $\beta_2$)} & (0.9, 0.999) \\
                \bottomrule
            \end{tabular}
    \end{center}
\end{table}

\paragraph{Training data.} The training data is sampled \textit{online}. Specifically, in each iteration, we sample $N_1$ i.i.d. data points, $(x^{(1)},t^{(1)}), \cdots, (x^{(N_1)},t^{(N_1)}) $, from the domain $\sR^n\times [0,T]$, and $N_2$ i.i.d. data points, $(\tilde x^{(1)},T), \cdots, (\tilde x^{(N_2)},T)$, from the boundary $\sR^n\times \{T\}$, where $(x^{(i)},t^{(i)}) \sim \mathcal{N}(\vzero,\mI_{n})\times \mathcal{U}(0,1)$ and $\tilde x^{(j)} \sim \mathcal{N}(\vzero,\mI_{n})$.

\vfill
\end{document}